\documentclass{article}
\pdfoutput=1

     \usepackage[nonatbib,preprint]{neurips_2019}

\usepackage[utf8]{inputenc} 
\usepackage[T1]{fontenc}    
\usepackage{hyperref}       
\usepackage{url}            
\usepackage{booktabs}       
\usepackage{amsfonts}       
\usepackage{nicefrac}       
\usepackage{microtype}      
\usepackage{amsmath}
\usepackage{amssymb}
\usepackage{amsthm}
\usepackage{subfigure}
\usepackage{graphicx}
 \usepackage{array,multirow}

\title{Lifted Regression/Reconstruction Networks}

\author{%
  Rasmus Kj\ae{}r H\o{}ier\thanks{This work was partially supported by the Wallenberg AI, Autonomous Systems and Software Program (WASP) funded by the Knut and Alice Wallenberg Foundation}
  \qquad {Christopher Zach}
  \\
  Chalmers University of Technology\\
  Gothenburg, Sweden \\
  \texttt{\{hier,zach\}@chalmers.se}
}

\theoremstyle{plain}

\newtheorem{corollary}{Corollary}
\newtheorem{lemma}{Lemma}
\newtheorem{theorem}{Theorem}
\theoremstyle{definition}
\newtheorem{remark}{Remark}

\providecommand{\norm}[1]{\lVert#1\rVert}
\providecommand{\Norm}[1]{\ensuremath{\left\lVert#1\right\rVert}}

\def\m#1{\ensuremath{\mathtt{#1}}}
\def\v#1{\ensuremath{\mathbf{#1}}}
\def\mI{\m I}

\def\sC{\ensuremath{\mathcal{C}}}
\def\cP{\ensuremath{\mathcal{P}}}

\begin{document}

\maketitle

\begin{abstract}
  In this work we propose \emph{lifted regression/reconstruction networks}
  (LRRNs), which combine lifted neural networks with a guaranteed Lipschitz
  continuity property for the output layer. Lifted neural networks explicitly
  optimize an energy model to infer the unit activations and therefore---in
  contrast to standard feed-forward neural networks---allow bidirectional
  feedback between layers. So far lifted neural networks have been modelled
  around standard feed-forward architectures. We propose to take further
  advantage of the feedback property by letting the layers simultaneously
  perform regression and reconstruction. The resulting lifted network
  architecture allows to control the desired amount of Lip\-schitz continuity,
  which is an important feature to obtain adversarially robust regression and
  classification methods. We analyse and numerically demonstrate applications
  for unsupervised and supervised learning.
\end{abstract}

\section{Introduction}
\label{sec:intro}

Deep neural networks (DNNs) are very powerful and expressive tools in machine
learning to solve many classification and regression tasks. Due to their
expressiveness and highly non-linear properties, DNNs are generally very
sensitive to minor perturbations of the input, which makes them unreliable in
e.g.\ safety-critical applications. A quickly growing body of works aims to
obtain robust DNNs either by design or by a dedicated training approach such
as adversarial training. In this work the main goal is to obtain powerful
regression methods that are robust to input perturbations by design. This is
achieved by controlling the Lipschitz continuity of the input-to-output
mapping, which puts limits on the sensitivity of a mapping with respect to
bounded input perturbations (w.r.t.\ the Euclidean norm). Thus, robustness of
the prediction is guaranteed even for unseen test samples and need not to be
established via an expensive verification procedure.

We further step away from pure feed-forward architectures for neural networks,
but base our classification and regression approach on layered energy-based
models, which enable bidirectional feedback between layers when determining
the internal network activations. Feed-forward DNNs can be obtained from such
layered energy models as a limit case, hence energy-based models can be
considered as powerful as regular DNNs. To our knowledge the Lipschitz
properties of such energy-based models have not been considered in the
literature. We propose a simple energy model that guarantees non-expansiveness
of the mapping from input to output activations essentially by symmetrizing an
energy model. Consequently each layer in the underlying energy has a
regression and a reconstruction component. In contrast to existing literature
on Lipschitz continuous DNNs, no difficult-to-enforce constraints on the
weight matrices (such as orthonormality) are needed in our framework.

\section{Related Work}

\paragraph{Lifted DNNs}
Lifted neural networks introduce an explicit set of unknowns for the internal
network activations, and inference is performed by minimization (or
marginalization) w.r.t.\ the network activations. Hence, they are based on a
different computational model than regular feed-forward networks.  Lifted
neural networks can be traced back to two somewhat different origins. First,
they can be seen as instances of more general undirected energy-based models
rooted in (restricted) Boltzmann
machines~\cite{ackley1985learning,smolensky1986,hinton2002training} and a
contrastive learning paradigm, where the learning signal is induced by the
energy difference between fully and partially clamped visible units
(e.g.~\cite{movellan1991contrastive}). It has been demonstrated, that
back-propagation is a limit case of contrastive
learning~\cite{xie2003equivalence,scellier2017equilibrium,zach2019contrastive}
for appropriate layered energy models.

A more recent motivation for lifted networks is the ability of highly parallel
training procedures~\cite{carreira2014distributed}, which proposes a quadratic
relaxation for the feedforward computation in a DNN, and using modern
optimization methods such as ADMM~\cite{taylor2016training}. The ability to
use convex energy models for e.g.~ReLU networks is connected with
re-interpreting the ReLU operation as projection to the non-negative real
line~\cite{zhang2017convergent,askari2018lifted}. The convex energy models
used in our work imitate feedforward networks only in so called weak feedback
setting. \cite{gu2018fenchel,li2018lifted} propose (block-convex but not
jointly convex) lifted network models that aim to replicate the standard
feed-forward pass in a DNN for a general class of activation functions.

\paragraph{DNN Robustness}
The discovery of the intrinsic brittleness of predictions made by deep neural
networks~\cite{szegedy2014intriguing} has led to a significant amount of
research on better constructing adversarial inputs
(e.g.~\cite{goodfellow2014explaining,moosavi2016deepfool,kurakin2016adversarial,carlini2017towards,madry2017towards})
and making neural-network based classifiers more robust with respect to
adversarial perturbations (e.g.\ by using a robust training
loss~\cite{szegedy2014intriguing,madry2017towards,wong2018provable} or
Lipschitz regularization~\cite{cisse2017parseval,yoshida2017spectral,tsuzuku2018lipschitz,qian2018l2}).
Determining adversarial perturbation amounts to minimizing a highly non-linear
and non-convex optimization problem, and therefore an explicit search for
adversarial examples cannot certify robustness of the network. Recently, it
was empirically shown, that adversarial training is insufficient by using
computationally expensive attacks~\cite{qian2018l2,wang2018mixtrain} and
therefore leads to a false sense of robustness.
These results strongly motivate the design of intrinsically robust neural
networks architectures.
Scattering networks~\cite{mallat2012group,bruna2013invariant} are
wavelet-based, non-trainable feature representations combining Lipschitz
continuity with transformation invariance.
Parseval networks~\cite{cisse2017parseval} aim for intrinsic 1-Lipschitz
continuity of a trained DNN by favoring orthonormal weight
matrices. Non-expansive networks~\cite{qian2018l2} propose to utilize
(approximately) distance preserving network layers, and Lipschitz margin
training~\cite{tsuzuku2018lipschitz} estimates the Lipschitz constant during
the training phase and uses it to adjust a required classification margin.

A complementary approach is to verify robustness of a
DNN for a particular input sample via robust optimization techniques.
Networks with general piece-wise linear activation functions can be verified
using linear programming relaxations~\cite{wong2018provable,weng2018towards},
which emerge immediately from exact, but not scalable, mixed integer-linear
programs~\cite{katz2017reluplex,ehlers2017formal}.
Stronger (but computationally demanding) relaxations can be obtained via
semi-definite programming~\cite{raghunathan2018certified,raghunathan2018semidefinite}.

\section{Lifted Regression/Reconstruction Networks (LRRN)}

In this section we propose a lifted network energy that is by construction
Lipschitz continuous with a user-specified Lipschitz constant. In contrast to
lifted networks proposed
in~\cite{zhang2017convergent,gu2018fenchel,li2018lifted,zach2019contrastive}
aiming to mimic the behavior of feed-forward DNNs, we add a reconstructive
term to the network energy model. Thus, we propose to use a network energy of
the form
\begin{align}
  \label{eq:LRRN}
  E(z;x) &= \frac{1}{2} \sum\nolimits_{k=0}^{L-1} \left( \Norm{z_{k+1}- W_kz_k - b_k}^2 + \beta_k \Norm{W_k^T z_{k+1} - z_k - c_k}^2 \right)
\end{align}
subject to $z_0=x$ and $z_k \in \sC_k$ for $k=1,\dotsc,L$. Each
$\sC_k\subseteq \mathbb{R}^{d_k}$, $k=1,\dotsc,L$, is a closed convex set,
which can be used to introduce non-linear behaviour. The choices
$\sC_k=\mathbb{R}^{d_k}$ (linear activation function) and
$\sC_k=\mathbb{R}_{\ge 0}^{d_k}$ (ReLU-like activation function) are of
particular interest.

Obtaining a prediction from the energy models requires computing the
activations $z^*(x) = \arg\min_z E(z;x)$ by solving a (strictly) convex
program. The last layer $z_L$ is considered the output layer, hence the
mapping $x\mapsto z_L^*(x)$ corresponds to the network's prediction function.

The parameters $\beta_k\ge 0$, $k=1,\dotsc,L$, control the Lipschitz constant
of $z_L^*(x)$ as we will see shortly. Setting $\beta_k=0$ for all $k$ yields a
pure forward regression network resembling standard feed-forward
DNNs~\cite{zach2019contrastive}.
We also tie the forward (regression) weights $W_k$ and the reconstructive
weights $W_k^T$, but not the biases $b_k$ and $c_k$ as tying them has no
impact on Lipschitz continuity.

\subsection{Motivation: Lipschitz continuity of linear 1-layer LRRNs}
\label{sec:motivation}

We consider first an LRRN with a single layer and no constraint on the latent
variables. Thus, the energy model in Eq.~\ref{eq:LRRN} reduces to
\begin{align}
  E(z; x) &= \tfrac{1}{2} \Norm{z-Wx-b}^2 + \tfrac{\beta}{2} \Norm{W^Tz - x - c}^2.
\end{align}
The first order optimality condition for $z^*$ for given $x$ is
\begin{align}
  \v 0 &= z - Wx - b + \beta W(W^Tz - x - c) \nonumber
\end{align}
or
\begin{align}
  (\m I + \beta WW^T) z^* &= Wx + b + \beta Wx + \beta Wc = (1+\beta) Wx + b + \beta Wc \nonumber.
\end{align}
Thus, $z^*(x)$ is explicitly given by
\begin{align}
  z^*(x) := (1+\beta) (\mI + \beta WW^T)^{-1} Wx + \underbrace{(\mI + \beta WW^T)^{-1} (b + \beta Wc)}_{=: \tilde c}.
\end{align}
Using the singular value decomposition of $W=U\Sigma V^T$, and therefore
$\mI + \beta WW^T = U (\mI + \beta \Sigma^2 )U^T$, this translates to
\begin{align}
  z^*(x) &= (1+\beta) U (\mI + \beta \Sigma^2 )^{-1} U^T U \Sigma V^T x + \tilde c \nonumber \\
  &= U (1+\beta) (\mI + \beta \Sigma^2 )^{-1} \Sigma V^T x + b =: A_\beta x + \tilde c.
\end{align}
The diagonal matrix $(1+\beta)(\mI + \beta \Sigma^2 )^{-1} \Sigma$ has the
elements $(1+\beta)\sigma_i/(1+\beta\sigma_i^2) \ge 0$ on its main
diagonal. The mapping $f_\beta:\mathbb{R}_{\ge 0} \to \mathbb{R}_{\ge 0}$ with
\begin{align}
  f_\beta(\sigma) &:= \frac{(1+\beta)\sigma}{1+\beta\sigma^2} \nonumber
\end{align}
has a single maximum at $\sigma = \sqrt{1/\beta}$.
Thus,
\begin{align}
  f_\beta(\sqrt{1/\beta}) &= \frac{(1+\beta)\beta^{-1/2}}{1+\beta\beta^{-1}} = \frac{(1+\beta)\beta^{-1/2}}{2}
  = \frac{\beta^{1/2} + \beta^{-1/2}}{2}\nonumber
\end{align}
This means that the singular values of $A_\beta$ are in
$[0,(\beta^{1/2} + \beta^{-1/2})/2]$, and the mapping
$x \mapsto A_\beta x + \tilde c$ is Lipschitz continuous with constant
$(\beta^{1/2} + \beta^{-1/2})/2$ (which is an upper bound on the operator norm
of $A_\beta$). With $\beta=1$ one has $\norm{A_\beta}_2 \le 1$.

\begin{remark}
  \label{rem:explicit_constant}
  The operator norm of $A_\beta$ can be explicitly stated as $\max_i\{
  \frac{(1+\beta)\sigma_i}{1+\beta\sigma_i^2} \}$, where $(\sigma_i)_i$ are
  the singular values of $W$. The maximum is attained for the singular value
  $\sigma_i$ that is ``closest'' (in a certain sense) to $1/\beta$. If
  $\beta\to 0$, then the largest of $\{\sigma_i\}$ yields the operator
  norm. Therefore in this setting the stated Lipschitz constant is explicitly
  dependent on $\beta$ and on the singular values $(\sigma_i)_i$.
\end{remark}

\subsection{Lipschitz continuity of proximal-like operators}

In the previous section the Lipschitz continuity of the mapping $x\mapsto
\arg\min_z \norm{z-Wx-b}^2/2 + \beta\norm{W^Tz-x-b}^2/2 = (1+\beta)(\mI+\beta
W^TW)^{-1}Wx + \tilde c$ was established. In order to add constraints on $z$
(such as non-negativity constraints to obtain a non-linear mapping) and to
analyse deeper LRRNs we need a more general approach. We are intersted in the
Lipschitz properties of the function
\begin{align}
  \label{eq:gen_prox}
  \cP_{\beta,W,G}: x \mapsto \arg\min_z \tfrac{1}{2} \norm{z-Wx}^2 + \tfrac{\beta}{2} \norm{W^T z - x}^2 + G(z),
\end{align}
where $G(z)$ is essentially an arbitrary convex function (not necessarily
differentiable). Since $\norm{z-Wx}^2/2$ is strictly convex in $z$, the
minimizer in the l.h.s.\ of Eq.~\ref{eq:gen_prox} is unique and therefore
$\cP_{\beta,W,G}$ is a proper function. If $W=\mI$, then the above mapping
reduces to
\begin{align}
  x \mapsto \arg\min_z (1+\beta)\norm{z-x}^2 + G(z),
\end{align}
which is known as proximal operator $x\mapsto
\operatorname{prox}_{G/(1+\beta)}(x)$ in the convex optimization
literature. Proximal operators are firmly non-expansive and therefore
1-Lipschitz continuous. This property is extended in a suitable way to
$\cP_{\beta,W,G}$:

\begin{lemma}
  \label{lem:lipschitz}
  Let $G$ be any proper l.s.c.\ convex function, $W\in\mathbb{R}^{n\times m}$
  a matrix with compatible dimensions, and $\beta\ge 0$. Then
  $\cP_{\beta,W,G}$ is $\frac{\beta^{1/2}+\beta^{-1/2}}{2}$-Lipschitz.
\end{lemma}
\begin{proof}
  The optimal $z^*=z^*(x)$ of $\norm{z-Wx}^2/2 + \beta\norm{W^T z - x}^2/2 +
  G(z)$ is determined by the optimality condition
  \begin{align}
    0 &= (\mI + \beta WW^T) z^* - (1+\beta)Wx + g,
  \end{align}
  where $g\in\partial G(z^*)$ is a subgradient of $G$ at $z^*$. Since the
  subgradient is a monotone operator, it satisfies
  \begin{align}
    (g_1 - g_2)^T(z_1 - z_2) \qquad \forall g_i \in \partial G(z_i).
  \end{align}
  Choosing $z_i = z_i^* = z^*(x_i)$ and inserting $g_i = (1+\beta)Wx_i - (\mI
  + \beta WW^T) z_i^*$ yields
  \begin{align}
    \big( (1+\beta)W(x_1-x_2) - (\mI + \beta WW^T)(z_1^*-z_2^*) \big) (z_1^* - z_2^*) \ge 0
  \end{align}
  or
  \begin{align}
    (1+\beta)(x_1-x_2)^T W^T (z_1^*-z_2^*) \ge (z_1^* - z_2^*)^T (\mI + \beta WW^T) (z_1^* - z_2^*).
  \end{align}
  We introduce $u:= x_1-x_2$ and $v:= z_1^* - z_2^*$. Among all $u$ with a
  fixed Euclidean norm $\delta\ge 0$, the vector $u$ leading to the largest
  l.h.s.\ is given by $u = \alpha W^T v$ (using the Cauchy-Schwarz
  inequality), where $\alpha \ge 0$ satisfies
  $\alpha\norm{W^T v} = \delta$. Hence, the above constraint can be restated
  as
  \begin{align}
    (1+\beta) \alpha v^T WW^T v \ge v^T (\mI + \beta WW^T) v & & \text{or} & & \norm{v}^2 \le \big( (1+\beta)\alpha-\beta \big) \norm{W^Tv}^2.
  \end{align}
  This induces the constraint $\alpha\ge \beta/(1+\beta)$ for the solution to
  be feasible. Inserting $u=W^Tv/\alpha$ and rearranging yields
  \begin{align}
    \frac{\norm{v}}{\norm{u}} = \frac{\norm{v}}{\alpha\norm{W^Tv}} &\le \frac{\sqrt{(1+\beta)\alpha-\beta}}{\alpha}
  \end{align}
  for all feasible $\alpha \ge \beta/(1+\beta)$. It is straightforward to
  verify that the mapping
  $f_\beta(\alpha) := \frac{\sqrt{(1+\beta)\alpha-\beta}}{\alpha}$ with domain
  $[\beta/(1+\beta),\infty)$ has range
  $[0,(\beta^{1/2}+\beta^{-1/2})/2]$, where the maximum is attained at
  $\alpha^*=2\beta/(1+\beta)$. Hence,
  \begin{align}
    \frac{\norm{z_1^*-z_2^*}}{\norm{x_1-x_2}} \le \frac{\beta^{1/2}+\beta^{-1/2}}{2} & & \text{i.e.}
    & & \norm{z_1^*-z_2^*} \le \frac{\beta^{1/2}+\beta^{-1/2}}{2} \norm{x_1-x_2},
  \end{align}
  which completes the proof.
\end{proof}
Consistent with Section~\ref{sec:motivation} the smallest Lipschitz constant
is obtained by setting $\beta=1$, which yields the following corollary:
\begin{corollary}
  $\cP_{1,W,G}$ is 1-Lipschitz.
\end{corollary}
Unlike in Section~\ref{sec:motivation} (cf.\
Remark~\ref{rem:explicit_constant}) the provided Lipschitz constant only
depends on $\beta$, and for $\beta\to 0$ the Lemma above yields a vacuous
bound. Nevertheless, one has the following simple lemma:
\begin{lemma}
  \label{lem:beta0}
  $\cP_{0,W,G}$ is $\norm{W}_2$-Lipschitz.
\end{lemma}
\begin{proof}
  We have
  \begin{align}
    \cP_{0,W,G}(x) = \arg\min_z \tfrac{1}{2} \norm{z-Wx}^2 + G(z) = \operatorname{prox}_G(Wx)
    = \left( \operatorname{prox}_G \circ (W\cdot) \right) (x),
  \end{align}
  i.e.\ $\cP_{0,W,G}(x)$ is the composition of a linear mapping with a
  proximal step. Since the Lipschitz constant of the mapping $x\mapsto Wx$ is
  $\norm{W}_2$ and the proximal operator is 1-Lipschitz, we deduce that the
  Lipschitz constant of $\cP_{0,W,G}$ is at most $\norm{W}_2$.
\end{proof}
In practice we are mostly interested in the choice of $\beta=1$ (both
regression and full reconstructive terms) and $\beta=0$ (pure regression term).

\subsection{General LRRNs}

Using Lemma~\ref{lem:lipschitz} the analysis of the energy underlying the
lifted regression/reconstruction networks (Eq.~\ref{eq:LRRN}) is relatively
straightforward. We define for $k=1,\dotsc,L$
\begin{align}
  \rho_k :=
  \begin{cases}
    \frac{\beta_k^{1/2}+\beta_k^{-1/2}}{2} & \text{if } \beta_k > 0 \\
    \norm{W_k}_2 & \text{if } \beta_k = 0.
  \end{cases}
\end{align}

\begin{theorem}
  For the layered energy model given in Eq.~\ref{eq:LRRN} let $z^*(x) =
  \arg\min_{z} E(z, x)$ be the minimizer for given input $x$. Then $x\mapsto
  z_L^*(x)$ (i.e.\ the mapping from the input to the last layer latent
  variables) is Lipschitz continuous with constant $\prod_{k=1}^L \rho_k$.
\end{theorem}
\begin{proof}
  Let $z_1^*(x)$ be given by
  \begin{align}
    z_1^*(x) &= \arg\min_{z_1} \min_{z_2,\dotsc,z_L} E((z_1,\dotsc,z_L), x) \nonumber \\
    &= \arg\min_{z_1\in\sC_1} \min_{z_2\in\sC_2,\dotsc,z_L\in\sC_L}
    \tfrac{1}{2} \sum \Norm{z_{k+1}- W_kz_k - b_k}^2 + \tfrac{\beta_k}{2} \sum \Norm{z_k - W_k^T z_{k+1} - c_k}^2.
  \end{align}
  Since minimizing out variables in a jointly convex function yields a convex
  function in the remaining unknowns, we can write the above as
  \begin{align}
    z_1^*(x) &= \arg\min_{z_1\in\sC_1} \tfrac{1}{2} \Norm{z_1- W_0x}^2 + \tfrac{\beta_1}{2} \Norm{W_0^T z_1 - x}^2 + G_1(z_1)
    = P_{\beta_1,W_0,G_1}(x).
  \end{align}
  Hence, $z_1^*(x)$ is $\rho_1$-Lipschitz due to Lemmas~\ref{lem:lipschitz}
  and~\ref{lem:beta0}. Due to the layered structure $z_2^*$ only depends on
  $z_1^*=z_1^*(x)$, therefore
  \begin{align}
    z_2^*(z_1^*) &= \arg\min_{z_2\in\sC_2} \tfrac{1}{2} \Norm{z_2- W_1z_1^*}^2 + \tfrac{\beta_2}{2} \Norm{W_1^T z_2 - z_1^*}^2 + G_2(z_2)
    = P_{\beta_2,W_1,G_2}(z_1^*)
  \end{align}
  for a suitable convex function $G_2$. Hence, $z_2^*(z_1^*)$ is
  $\rho_2$-Lipschitz. Applying this argument iteratively on the remaining
  layers, we find that $z_k^*(z_{k-1}^*)$ is $\rho_k$-Lipschitz. Further,
  $z_L^*(x) = \big( z_L^* \circ z_{L-1}^* \circ \cdots \circ z_1^*\big)(x)$,
  the Lipschitz constant of $z_L^*(x)$ is at most $\prod_{k=1}^L \rho_k$.
\end{proof}

\begin{corollary}
  Let $\beta_1=\cdots=\beta_{L-1}=1$ and $\beta_L=0$ (i.e.\ the output layer
  is a pure regression layer). Then the Lipschitz constant of $x\mapsto
  z_L^*(x)$ is at most $\norm{W_{L-1}}_2$.
\end{corollary}
Networks with such a choice for $(\beta_k)_{k=1}^L$ have a clear
interpretation: the first $L-1$ layers extract non-expansive feature
representations, and the last layer is an arbitrary linear regression layer to
generate the target output. Hence, the Lipschitz properties of the network
(and therefore the robustness with respect to input perturbations) can be
assessed by inspecting the last layer matrix $W_{L-1}$.

\begin{remark}
  The network energy in Eq.~\ref{eq:LRRN} allows direct feedback from a
  subsequent layer to the previous one (and therefore indirect feedback to all
  earlier layers). This feedback from later layers can be essentially
  suppressed by using discounted terms~\cite{xie2003equivalence,zhang2017convergent},
  \begin{align}
    \label{eq:discounted_LRRN}
    E(z;x) &= \frac{1}{2} \sum\nolimits_{k=0}^{L-1} \gamma^{k-1} \left( \Norm{z_{k+1}- W_kz_k - b_k}^2 + \beta_k \Norm{W_k^T z_{k+1} - z_k - c_k}^2 \right)
  \end{align}
  for a \emph{feedback parameter} $\gamma>0$. With $\gamma\to 0$ one recovers
  a feed-forward DNN, and the contrastive learning approach for supervised
  training (see Section~\ref{sec:supervised}) is equivalent to
  back-propagation. In that sense energy-based models such as
  Eqs.~\ref{eq:LRRN} and~\ref{eq:discounted_LRRN} are more general than
  feed-forward networks. Observe that Eq.~\ref{eq:LRRN} and
  Eq.~\ref{eq:discounted_LRRN} are actually equivalent, since the feedback
  weight $\gamma^{k-1}$ can be absorbed by reparametrizing the weights
  $W_k$, biases $b_k$ and $c_k$, and the activations $z_k$. Nevertheless, the
  feedback parameter still influences initialization of the network parameters
  and any weight regularization term.
\end{remark}

\paragraph{Implementing LRRNs}
Determining $z^*(x)$ requires minimizing a strictly (even strongly) convex,
possibly constrained, optimization problem Eq.~\ref{eq:LRRN}. The easiest
method to solve such a task is coordinate descent, which minimizes
sequentially the scalar network activations $\{z_{kj}\}$ (with $k$ iterating
over all layers and $j$ traversing units in the current layer). For many
relevant constraint sets $\sC_k$ the optimal solution for $z_{kj}$ after
fixing all other activations can be stated in closed form. Hence, we employ
coordinate descent in our implementation.

\section{Learning with LRRNs}

In this section we briefly discuss the application of LRRNs for unsupervised
and supervised learning. In the unsupervised setting LRRNs generalize subspace
learning, and supervised learning requires a non-standard approach since
back-propagation is not directly applicable for energy-based network models.
We show results for the MNIST~\cite{lecun1998gradient}, Fashion-MNIST
(FMNIST,~\cite{xiao2017fashion_mnist}) and
Kuzushiji-MNIST~(KMNIST,~\cite{clanuwat2018deep}) datasets. Training of the
models is performed by stochastic gradient descent, where the activations
$z^*$ are first inferred using coordinate descent, and the contribution of a
training sample to the gradient is based on these activations, e.g.\
$\nabla_{W_k} E(z^*; x) = (W_kz_k^*-z_{k-1}^*+b_k) (z_k^*)^T + \beta_k
z_{k+1}^* (W_k z_{k+1}^* - z_k^* -c_k)^T$.

\subsection{Unsupervised setting}

Let $\{x_i\}_{i=1}^N$ be a set of unlabelled training samples. One question is whether
a sensible energy model can be obtained by solely minimizing the average
energy of the training samples, i.e.
\begin{align}
  \label{eq:unsupervised}
  \min_\theta J(\theta) &= \min_\theta \tfrac{1}{N} \sum\nolimits_i \min_z E(z; x_i) = \min_\theta \tfrac{1}{N} \sum\nolimits_i E(z^*(x_i); x_i),
\end{align}
where $\theta$ contains all the weights and biases in the energy model
(Eq.~\ref{eq:LRRN}). Since in this setting we are not interested in the output
layer, $E(z; x)$ reduces to
\begin{align}
  \label{eq:unsupervised1}
  E(z_1;x) = \tfrac{1}{2} \Norm{z_1-W_0x}^2 + \tfrac{1}{2} \Norm{W_0^Tz_1-x}^2 + G_1(z_1),
\end{align}
where $G_1$ is a convex function obtained by minimizing out all subsequent
layers $z_2,\dotsc,z_L$, and $G_1$ acts therefore as a (learnable) prior on
$z_1$. We also absorbed the bias terms into $G_1$.

Since $E(z;x)$ induces an \emph{unnormalized} energy model $E(x) = \min_z
E(z;x)$, it is not immediately clear that the loss in
Eq.~\ref{eq:unsupervised} (which can also be seen as maximizing an
unnormalized probability) leads to any desired energy model. It can be shown
that if $G_1$ is coercive (i.e.\ $G_1(z_1) \to\infty$ as $\norm{z_1}\to
\infty$), then $E(x)$ is also coercive, and therefore samples with small
energy (thus highly likely ones) are concentrated to bounded (convex)
sets. Hence, proper shaping of $G_1$ is somewhat analogous to the ``bottleneck'' in
standard auto-encoders.

Note that Eq.~\ref{eq:unsupervised1} resembles an auto-encoder with a
single hidden layer: the first term defines the encoder, the second term is a
reconstruction error and therefore corresponds to the decoder, and the last
term represents the prior on the latent variables.
If $G_1(z_1)\equiv 0$ (which means that the underlying LRRN has exactly
one linear layer), it can be shown that Eq.~\ref{eq:unsupervised} essentially
performs an eigen-decomposition of the scatter matrix $\sum_i x_ix_i^T$, and
is therefore strongly connected to PCA and subspace learning. If all training
points $\{x_i\}$ lie in an $r$-dimensional subspace, and $\dim(z_1)=r$,
then $W_0^T$ is an orthonormal matrix satisfying $W_0W_0^T = \mI$, and
$E(z_1^*(x); x)=0$ for all points lying in that subspace.

Fig.~\ref{fig:Linear_32hidden} depicts the filter obtained from such
unsupervised training on the MNIST dataset. Using linear activations leads to
PCA-like modes for the filter (Fig.~\ref{fig:Linear_32hidden}(a)), whereas
ReLU-like non-negative activations yield filters that resemble dictionary
elements learned via sparse coding (Fig.~\ref{fig:Linear_32hidden}(b)). Unlike
PCA, the filters in Fig.~\ref{fig:Linear_32hidden}(a) describe only a subspace
and are therefore not necessarily aligned with the PCA basis. Further visual
results are shown in the appendix.

\begin{figure}[htb]
  \centering
  \subfigure[Linear]{\includegraphics[width=0.47\textwidth]{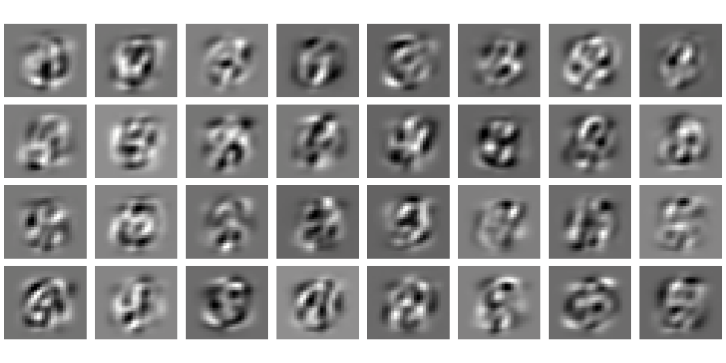}}~~~~~~
  \subfigure[ReLU-type]{\includegraphics[width=0.47\textwidth]{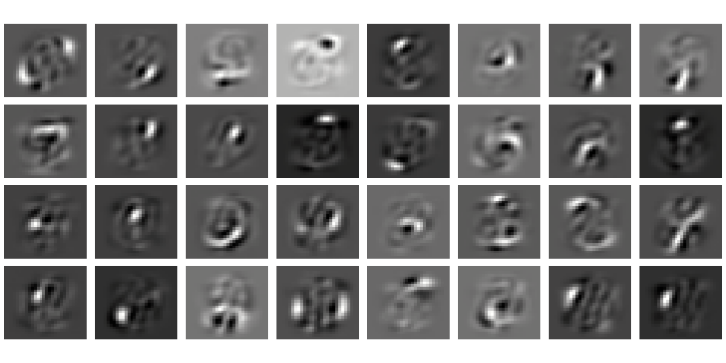}}
  \caption{First layer filters of unsupervised 784-32-32 LRRNs using different activations.}
  \label{fig:Linear_32hidden}
\end{figure}

Table~\ref{tab:E_unsup} demonstrates the ability of a trained energy-based
model to distinguish between different datasets. Test data from the same
dataset has on average consistently smaller energies (along the main diagonal)
than samples from different datasets. Further, horizontally flipped (mirrored)
test data from the same dataset has on average also a higher energy than the
original test data. Finally, samples from a Gaussian fitted to the training
set with diagonal covariance matrix have significantly higher energy
values. Note that the characters in~KMNIST have generally larger variablity
than e.g.\ MNIST, which explains the higher overall energies for this
dataset. In summary, unsupervised learning of LRRNs allows to train
\emph{unnormalized} energy models capturing the training distribution
\emph{even without} explicitly addressing the lack of normalization (as
opposed to contrastive divergence~\cite{hinton2002training}, noise-contrastive
estimation~\cite{gutmann2012noise} or score
matching~\cite{hyvarinen2005estimation}).

\begin{table}[htb]
  \centering
  \begin{tabular}{c|l|c|c|c|c|c}
    &                & \multicolumn{5}{c}{Test data}                                  \\ \hline
    & $\min_zE(z;x)$ & MNIST   & KMNIST  & FMNIST &  Mirrored & Fitted Gaussian \\ \hline
    \parbox[t]{2mm}{\multirow{3}{*}{\rotatebox[origin=c]{90}{Training}}}
    & MNIST  & 13.7$\pm$5.1  & 26.4$\pm$6.0  & 35.9$\pm$8.3 & 21.5$\pm$6.8~ & 48.0$\pm$3.8 \\ \cline{2-7} 
    & KMNIST & 59.5$\pm$24.1 & 36.5$\pm$12.9 & 52.1$\pm$16.0 & 43.8$\pm$15.6 & 78.9$\pm$4.5 \\ \cline{2-7} 
    & FMNIST & 55.4$\pm$30.8 & 26.8$\pm$11.5 & 12.6$\pm$6.7~ & 17.4$\pm$11.0 & 65.0$\pm$3.9
  \end{tabular}
  \caption{Avg.~energies (and std.~deviations) of unsupervised 784-32-32 ReLU-models trained
    on MNIST, KMNIST and FMNIST (rows), evaluated on different test sets (rows).
  }
  \label{tab:E_unsup}
\end{table}

\begin{remark}
  Since $E(z_1;x)$ is jointly convex in $x$ and $z_1$, $E(z_1^*(x);x)$ reduces
  to a convex function of $x$. Due to the connection of $\cP_{1,W_0,G_1}$ with
  proximal operators (which are themselves generalizations of projection steps
  to convex sets), the mapping $x\mapsto E(z_1^*(x);x)$ can be interpreted as
  generalization of the squared distance to a convex set (whose exact shape is
  learned from data). Hence, $E(z_1^*(x);x)$ cannot directly represent e.g.\
  non-convex manifolds.
  In order to increase the expressive power of LRRNs in the unsupervised
  setting, one can e.g.\ train class-specific energy models with weights
  shared across classes,
  %
  \begin{align}
    \label{eq:class_specific}
    \tfrac{1}{N} \sum\nolimits_i \min_{z:z_L=y_i} E(z; x_i) \to \min\nolimits_\theta,
  \end{align}
  where $y_i$ is the label associated with $x_i$.
\end{remark}

\subsection{Supervised learning}
\label{sec:supervised}

Let $\{(x_i,y_i)\}_{i=1}^N$ be labelled training data, and the aim is to
estimate network parameters such that $z_L^*(x_i) \approx y_i$, where
$z_L^*(x)$ is obtained by minimizing the energy model in Eq.~\ref{eq:LRRN}.
Since the mapping $x_i \mapsto z_L^*(x_i)$ has generally no closed form
expression in terms of the model parameters
$\theta=(W_k,c_k,b_k)_{k=0}^{L-1}$, we employ a contrastive learning
approach~\cite{xie2003equivalence,scellier2017equilibrium,zach2019contrastive}
in the supervised setting. Let us denote the free and the so called clamped
solution by $z^*(x)$ and $\hat z(x,y)$, respectively,
\begin{align}
  z^*(x) = \arg\min\nolimits_z E(z; x) & & \hat z(x,y) = \arg\min\nolimits_{z:z_L=y} E(z; x).
\end{align}
Thus, the clamped solution is obtained by minimizing the network energy with
the additional constraint of fixing the output layer. By construction
$E(z^*(x);x) \le E(\hat z(x,y);x)$, and the aim of contrastive learning is to
close the gap between these two energies by adjusting the model parameters
$\theta$:
\begin{align}
  \label{eq:contrastive}
  \ell(\theta) = \tfrac{1}{N} \sum\nolimits_i \big( E(\hat z(x_i,y_i); x_i) - E(z^*(x_i); x_i) \big) \to \min\nolimits_\theta.
\end{align}
This loss can be interpreted as an approximation of the cross-entropy
loss~\cite{zach2019contrastive}. Since $E$ is strongly convex,
$E(\hat z(x,y); x) \approx E(z^*(x); x)$ implies that
$z_L^*(x)\approx y$. We apply weight decay regularization on the last layer
matrix $W_{L-1}$ to favor non-contracting, distance-preserving feature
representations in the first $L-1$ layers.

Fig.~\ref{fig:supervised_results} illustrates the evolution of the training
loss and test accuracies w.r.t.\ the number of epochs for the MNIST,
FMNIST and KMNIST datasets.
By inspecting the spectral norm of the last layer weight matrix, the mappings
$x\!\mapsto\!z_L^*(x)$ have Lipschitz constants of at most 0.94, 0.95, and 1.07,
respectively, for MNIST, KMNIST and FMNIST trained models. It can be easily
derived (e.g.~\cite{tsuzuku2018lipschitz}), that the classifier output is
unaffected by any input perturbation $\Delta x$ with
$\norm{\Delta x}_2 \le m_x/(\sqrt{2}\rho)$, where $\rho$ is a Lipschitz
constant of $x\!\mapsto\! z_L^*(x)$ and
$m_x = z_{L,j(x)}^*(x) - \max_{j\ne j(x)} z_{L,j}^*(x)$ is the margin for the
predicted label $j(x) = \arg\max_j z_{L,j}^*(x)$. Given the median margins for
the corresponding test data, this translates to median norms of 0.70 (MNIST),
0.51 (KMNIST), and 0.46 (FMNIST), for provably safe perturbations.
For comparison, \cite{tsuzuku2018lipschitz} reports a value of~1.02 for an
MNIST-trained small-scale CNN (but does not state its test accuracy).
We also explore the impact of unsupervised pretraining on supervised learning
in the appendix (yielding slightly higher accuracies).

\begin{figure}[htb]
  \centering
  \includegraphics[width=0.49\textwidth]{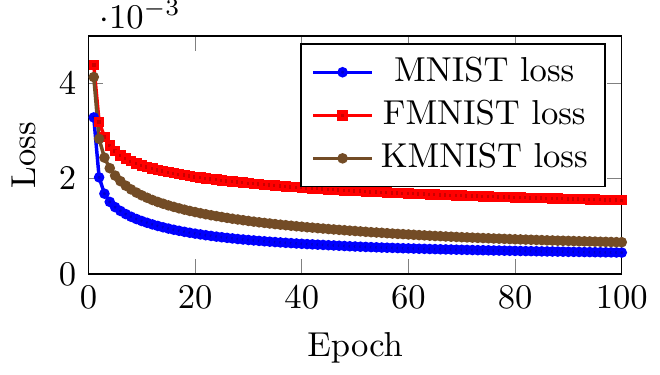}
  \includegraphics[width=0.49\textwidth]{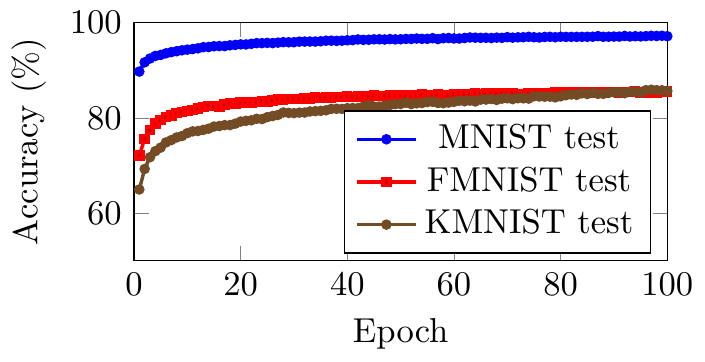}
  \caption{Loss and classification accuracies of 784-64-64-10 ReLU-type
    LRRNs. A final accuracy of 97.2\% is achieved for MNIST, 85.6\% for FMNIST
    and 85.7\% for KMNIST.}
  \label{fig:supervised_results}
\end{figure}

\section{Conclusion}

We propose lifted regression/reconstruction networks (LRRNs), that guarantee
controlled Lipschitz continuity with easy-to-evaluate constants in layered
energy-based models. This is achieved by essentially symmetrizing the terms in
the underlying energy model, and therefore explicit penalizers on the model
parameters (e.g.\ weight matrices) are not required to achieve a target
Lipschitz property. We demonstrate how LRRNs can be used for both supervised
and unsupervised learning.
Future work includes exploration of the semi-supervised setting by combining
the discriminative (contrastive) loss with unsupervised training. One goal is
to obtain a unified DNN architecture for regression (and classification) that
is further able to detect out-of-distribution samples (in the spirit
of~\cite{grathwohl2019your}), but at the same time explicitly allow regression
tasks, aim for robustness by design, and target a unified training method.

\small
\bibliographystyle{plain}
\bibliography{literature}

\appendix

\section{Unsupervised learning: additional visual results}
All networks were trained for 100 epochs with a learning rate
$\eta=0.005$.  Fig.~\ref{fig:unsup} shows filters obtained by unsupervised
training of a 784-32-32 network with linear activation
($\mathcal C_k = \mathbb{R}^{d_k}$), ReLU-like activation
($\mathcal C_k = \mathbb{R}_{\ge 0}^{d_k}$) and hard-sigmoid activation
function ($\mathcal C_k = [-1,1]^{d_k}$). Fig.~\ref{fig:unsup_ReLU} illustrates
the first layer filters for MNIST, KMNIST and FMNIST using the ReLU-type
activation function.

The impact of the chosen activation function and dataset on the visual
properties of the filters is evident. More constrained network activations
leads to sparser filters, and the visual appearance of each dataset is also
reflected in the filters.

\begin{figure*}[htb]
    \centering
    \subfigure[Linear activations]{\includegraphics[width=0.32\textwidth]{MNIST_unsup_Linear_filters.png}}~~
    \subfigure[ReLU activations]{\includegraphics[width=0.32\textwidth]{MNIST_unsup_ReLU_filters.png}}~~
    \subfigure[Hard sigmoid activations]{\includegraphics[width=0.32\textwidth]{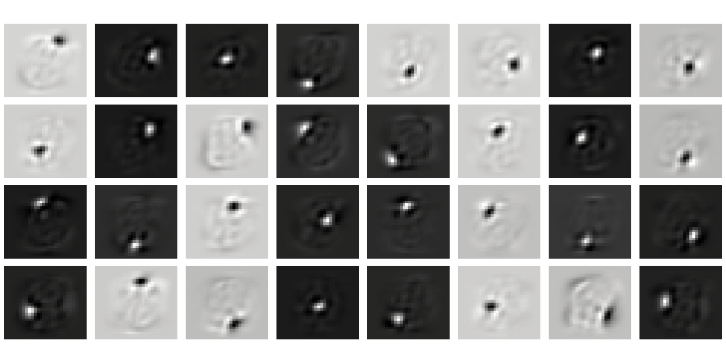}}
    \caption{Filters of 784-32-32 unsupervised networks trained on MNIST using different activation functions .}
    \label{fig:unsup}
\end{figure*}

\begin{figure*}[htb]
    \centering
    \subfigure[MNIST]{\includegraphics[width=0.32\textwidth]{MNIST_unsup_ReLU_filters.png}}~~
    \subfigure[KMNIST]{\includegraphics[width=0.32\textwidth]{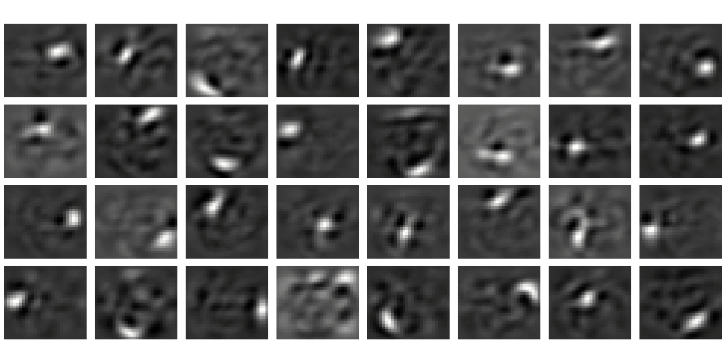}}~~
    \subfigure[FMNIST]{\includegraphics[width=0.32\textwidth]{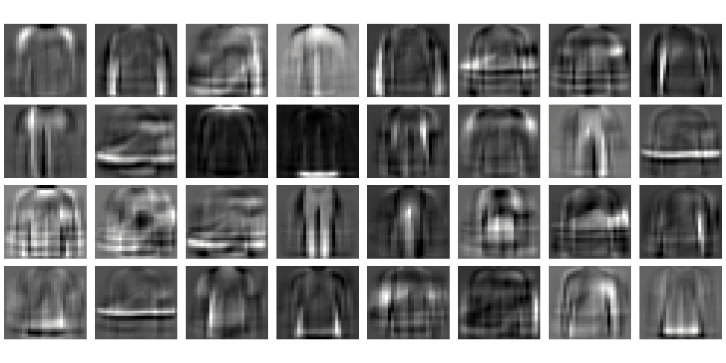}}
    \caption{Filters of 784-32-32 ReLU networks trained unsupervised on MNIST, KMNIST and FMNIST.}
    \label{fig:unsup_ReLU}
\end{figure*}

\section{Supervised Learning}

\subsection{Supervised learning from random initialization}

A 784-64-64-10 ReLU-type LRRN network was trained for 100 epochs, using 20 BCD
passes when inferring activations. A learning rate $\eta=0.4$ was used for MNIST
and KMNIST and $\eta=0.1$ was used for FMNIST. The reconstruction prefactors
were chosen as $\beta=[1, 1, 0]$ and the final layer was linear. The feedback
parameter was chosen as $\gamma=1/8$, and a batch size of 10 was used.

\begin{figure}[htb]
    \centering
    \subfigure[MNIST filters]{\includegraphics[width=0.32\textwidth]{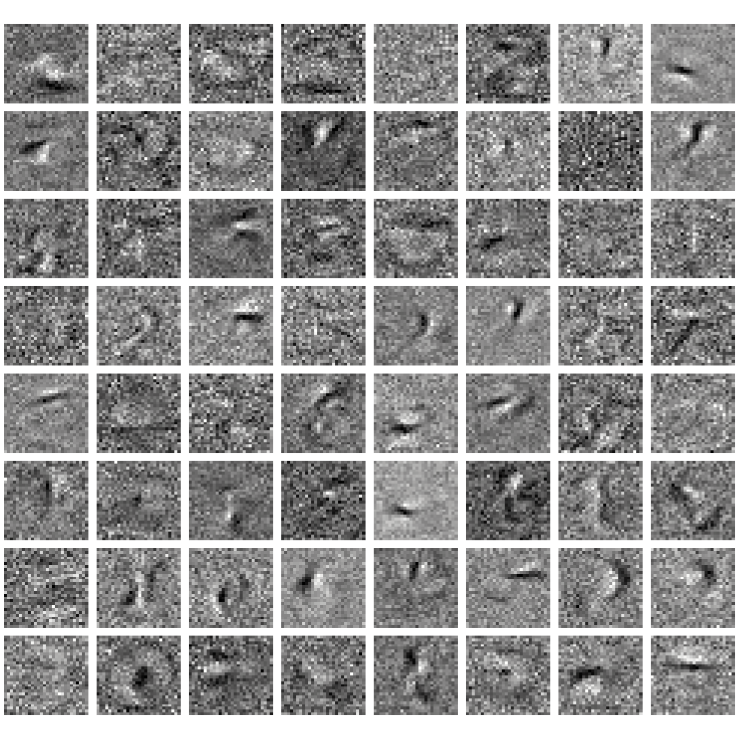}}~~
    \subfigure[KMNIST filters]{\includegraphics[width=0.32\textwidth]{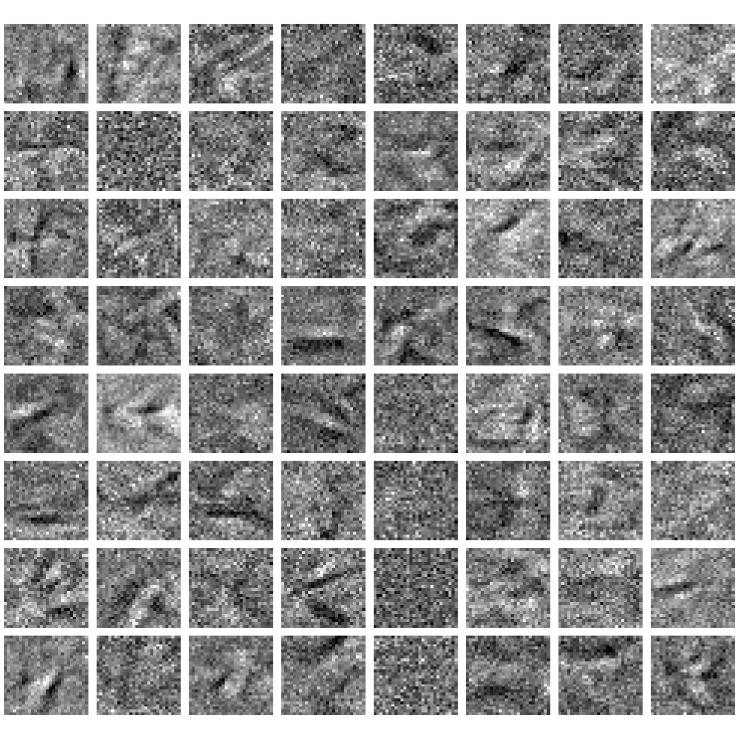}}~~
    \subfigure[FMNIST filters]{\includegraphics[width=0.32\textwidth]{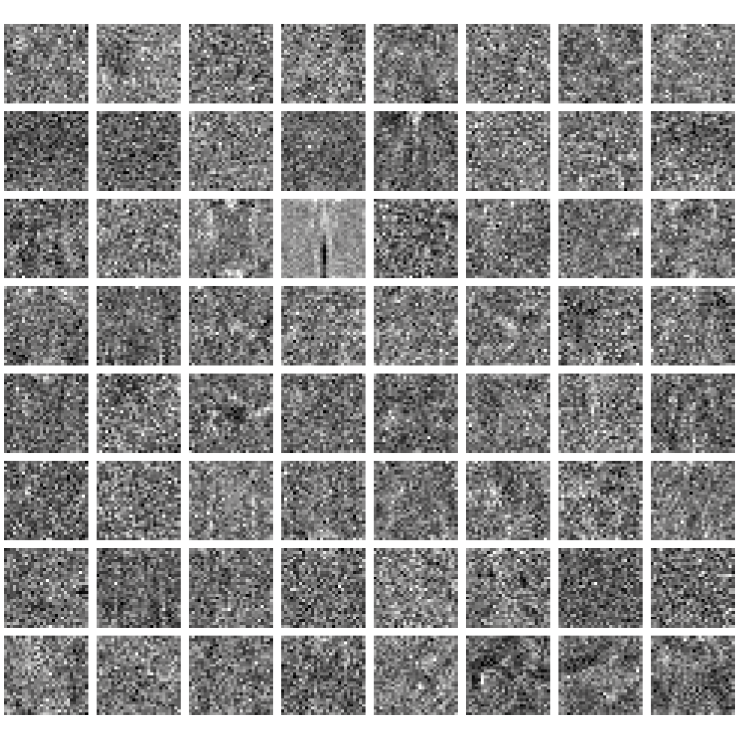}}
    \caption{784-64-64-10 ReLU-type LRRN first layer filters.}
\end{figure}

\subsection{Supervised learning with unsupervised pretraining}

Three 784-64-64-10 LRRNs were trained on MNIST, KMNIST and FMNIST in an
unsupervised manner by minimizing the free energy. A learning rate of
$\eta=0.005$ was used for all the networks. Furthermore $\beta = [1, 1, 0]$,
$\gamma = 1/8$, and a mini-batchsize of 10. The resulting filters are shown in
Fig.~\ref{fig:filters_unsup}. The networks were then trained with supervision
for additional 100 epochs. Fig.~\ref{fig:filters_finetuned} depicts the
fine-tuned first layer filters, which retained most of their interpretable
appearance. Fig.~\ref{fig:semi_pretrained_results} shows the learning progress
in terms of training loss and test accuracies. Pretraining leads to slightly
better classification results, especially for the KMNIST dataset (85.8\% vs.\
87.91\%).

\begin{figure}[htb]
  \centering
  \includegraphics[width=0.49\textwidth]{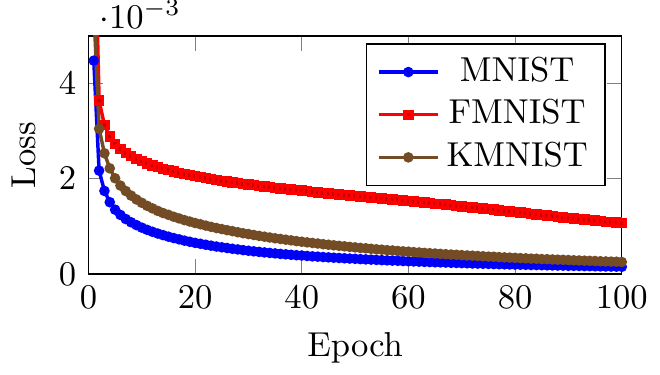}
  \includegraphics[width=0.49\textwidth]{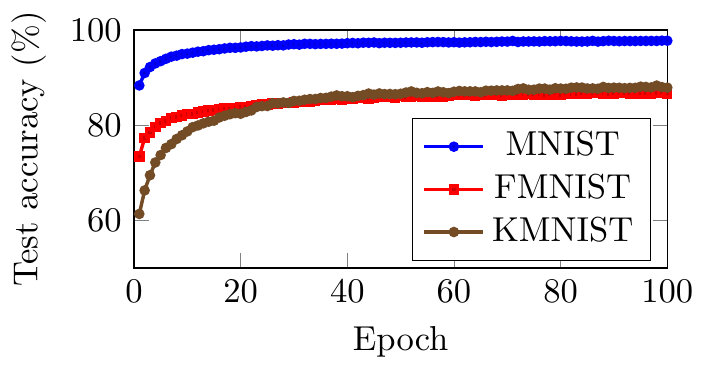}
  \caption{Loss and classification accuracies for the supervised training
    phase of a 784-64-64-10 ReLU-type LRRN, which has undergone unsupervised
    pretraining. An accuracy of 97.8\% is achieved for MNIST, 87.9\% for
    KMNIST, and 86.7\% for FMNIST}
  \label{fig:semi_pretrained_results}
\end{figure}

\begin{figure}[htb]
    \centering
    \subfigure[MNIST filters]{\includegraphics[width=0.32\textwidth]{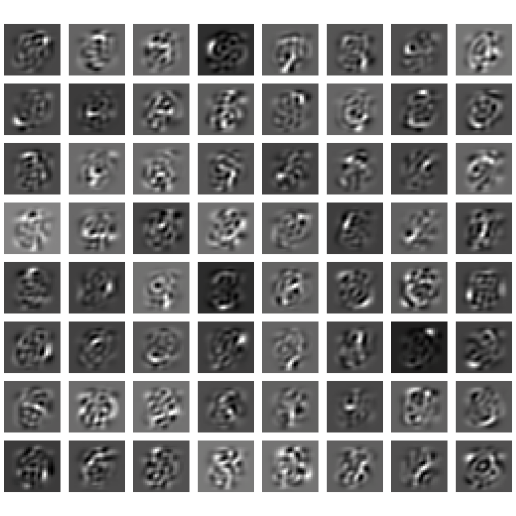}}~~
    \subfigure[KMNIST filters]{\includegraphics[width=0.32\textwidth]{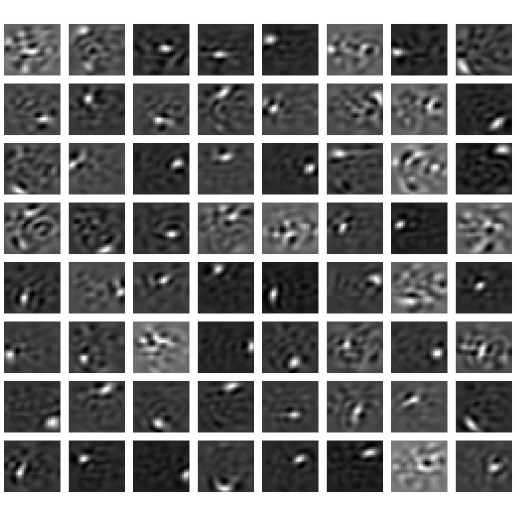}}~~
    \subfigure[FMNIST filters]{\includegraphics[width=0.32\textwidth]{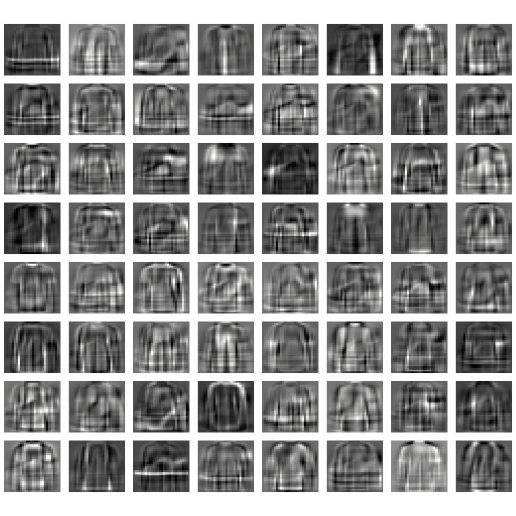}}

    \caption{First layer filters of 784-64-64-10 ReLU-LRRNs after 20 epochs of unsupervised training.}
    \label{fig:filters_unsup}
\end{figure}

\begin{figure}[htb]
  \centering
  \subfigure[MNIST filters]{\includegraphics[width=0.32\textwidth]{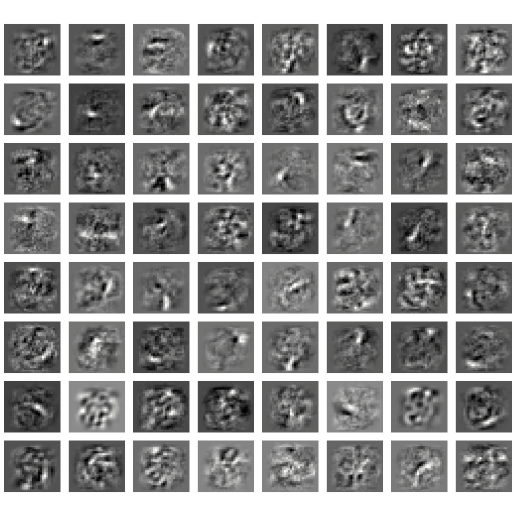}}~~
  \subfigure[KMNIST filters]{\includegraphics[width=0.32\textwidth]{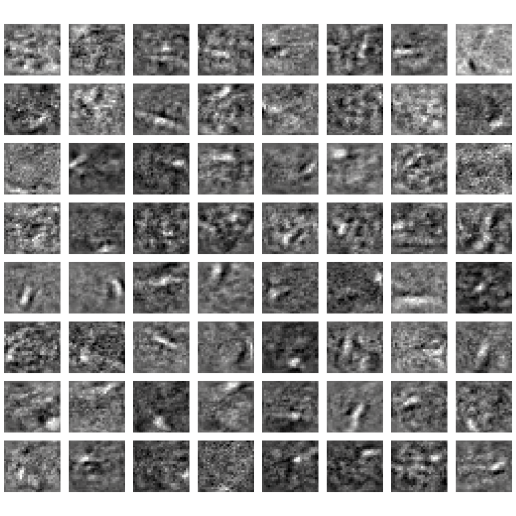}}~~
  \subfigure[FMNIST filters]{\includegraphics[width=0.32\textwidth]{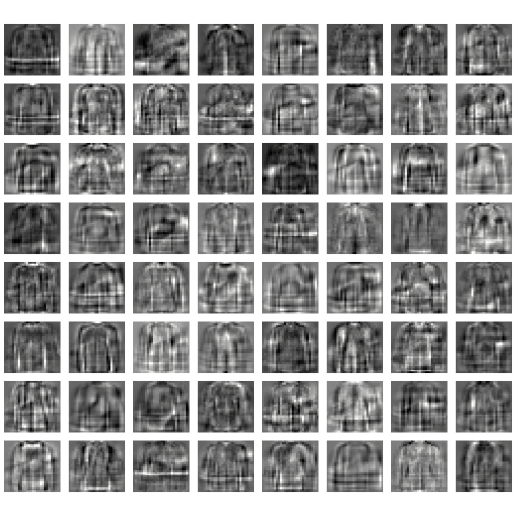}}

  \caption{First layer filters of the same 784-64-64-10 ReLU-LRRNs as shown in Fig.~\ref{fig:filters_unsup} after an additional 100 epochs of supervised training.}
  \label{fig:filters_finetuned}
\end{figure}

\subsection{The impact of weight decay on the Lipschitz estimates}

Tables~\ref{tab:decay1} and~\ref{tab:decay2} list the estimates for the
Lipschitz constants $\rho$, the classification margin $m$ and the allowed
$\ell_2$ norm $\delta$ for safe perturbations for two different weights on the
weight decay term ($5\times 10^{-5}$ and $5\times 10^{-4}$). The upper bound
on $\delta$ has been calculated via
$\delta \leq \frac{m}{\sqrt{2}\rho}$. Higher weight decay regularization
clearly induces a tradeoff between accuracy and perturbation robustness.

\begin{table}[htb]
  \centering
  \begin{tabular}{l|l|l|l|l|l|l}
    & $\rho$    & $mean(m)$ & $median(m)$ & $std(m)$ & Median $\ell_2$ norm $\delta$ & test accuracy \\ \hline
    MNIST  & 0.9387 & 0.8299    & 0.9270      & 0.2592   & 0.698 & 97.2\%  \\ \hline
    KMNIST & 0.9548 & 0.6195    & 0.6827      & 0.3442   & 0.506 & 85.6\%  \\ \hline
    FMNIST & 1.0716 & 0.6318    & 0.6913      & 0.3397   & 0.456 & 85.7\%
  \end{tabular}
  \caption{Lipschitz value of models trained on the three datasets (MNIST, KMNIST and FMNIST) and mean, median and standard deviation of classification margins. Weight decay factor $5\times 10^{-5}$.}
  \label{tab:decay1}
\end{table}

\begin{table}[htb]
  \centering
  \begin{tabular}{l|l|l|l|l|l|l}
    & $\rho$    & $mean(m)$ & $median(m)$ & $std(m)$ & Median $\ell_2$ norm $\delta$ & test accuracy \\ \hline
    MNIST  & 0.4604 & 0.6581  & 0.7038 & 0.2899 & 1.0810 & 95.5\% \\ \hline
    KMNIST & 0.4174 & 0.4402  & 0.4101 & 0.3051 & 0.6948 & 80.5\% \\ \hline
    FMNIST & 0.4142 & 0.5006  & 0.4776 & 0.3267 & 0.8153 & 82.4\%
  \end{tabular}
  \caption{Lipschitz value of models trained on the three datasets (MNIST, KMNIST and FMNIST) and mean, median and standard deviation of classification margins. Weight decay factor $5\times 10^{-4}$.}
  \label{tab:decay2}
\end{table}

\end{document}